\documentclass[letterpaper,11pt]{article}

\usepackage[paperwidth=8.5in,paperheight=11.0in,left=1.5in, right=1in, top=1in, bottom=1.25in, dvips]{geometry}
\usepackage{lmodern}
\usepackage{anyfontsize}
\usepackage{blindtext}
\usepackage{graphicx,float,psfrag,amsmath,amsthm,amssymb,nicefrac,mathtools}
\usepackage[english]{babel}
\usepackage{listings}
\usepackage{soul}
\usepackage{multirow}
\usepackage{algpseudocode}
\usepackage{algorithm,algpseudocode}
\usepackage{natbib}
\usepackage{enumitem}
\usepackage{url,color,booktabs,verbatim}
\usepackage[colorlinks=true,allcolors=black,linktoc=all,hyperindex=true]{hyperref}
\usepackage{etaremune}
\usepackage{titlesec}
\usepackage{pdfpages}

\usepackage{nomencl}
\makenomenclature

\usepackage[utf8]{inputenc}

\bibpunct[, ]{(}{)}{;}{a}{,}{,} 

\theoremstyle{plain}

\newtheorem{theorem}{Theorem}

\newtheorem{lemma}{Lemma}

\newtheorem*{prop*}{Proposition}

\theoremstyle{definition}

\newtheorem*{defn*}{Definition}

\newtheorem*{exmp*}{Example}

\newtheorem*{conj*}{Conjecture}

\theoremstyle{remark}

\newtheorem*{rmk*}{Remark}

\def \ifempty#1{\def\temp{#1} \ifx\temp\empty }

\newcommand{\q}{\ensuremath{\mathbf{q}}}

\newcommand{\y}{\ensuremath{\mathbf{y}}}

{%
\begin{list}{#1}{
\vspace{-\topsep}
\vspace{-\partopsep}
\setlength{\itemindent}{0cm}
\setlength{\rightmargin}{0cm}
\setlength{\listparindent}{0cm}
\settowidth{\labelwidth}{#1}
\setlength{\leftmargin}{\labelwidth}
\addtolength{\leftmargin}{\labelsep}
\setlength{\itemsep}{0cm}
}%
}%
{%
\end{list}
\vspace{-\topsep}
\vspace{-\partopsep}
}

{\begin{enumerate}%
}%
{\end{enumerate}}

\DeclareVoidOption{hidelinks}{%
 \Hy@colorlinksfalse
 \Hy@ocgcolorlinksfalse
 \Hy@frenchlinksfalse
 \def\Hy@colorlink##1{\begingroup}%
 \def\Hy@endcolorlink{\endgroup}%
 \def\@pdfborder{0 0 0}%
 \let\@pdfborderstyle\ltx@empty
 }

\renewcommand{\Re}{\mathbb{R}}
\usepackage{multirow}

\def\Y{{\cal Y}}

\def\D{{\cal D}}
\def\L{{\cal L}}

\def\O{{\cal O}}
\def\s{\mathbf{s}}
\def\y{\mathbf{y}}
\def\q{\mathbf{q}}
\def\r{\mathbf{r}}

\title{Deep Reinforcement Learning via L-BFGS Optimization}
\author{Jacob Rafati \\ 
\texttt{jrafatiheravi@ucmerced.edu}\\
Electrical Engineering and Computer Science\\
\and
Roummel F. Marcia\\
\texttt{rmarcia@ucmerced.edu}\\
Department of Applied Mathematics
\and
Univeristy of California, Merced.\\
5200 North Lake Road, Merced, CA, 95343, USA.}

\date{}
\begin{document}

\maketitle

\section*{Abstract}
Reinforcement Learning (RL) algorithms allow artificial agents to improve their action selections so as to increase rewarding experiences in their environments. The learning can become intractably slow as the state space of the environment grows. This has motivated methods like Q-learning to learn representations of the state by a function approximator. Impressive results have been produced by using deep artificial neural networks. However, deep RL algorithms require solving a nonconvex and nonlinear unconstrained optimization problem. Methods for solving the optimization problems in deep RL are restricted to the class of first-order algorithms, such as stochastic gradient descent (SGD). The major drawback of the SGD methods is that they have the undesirable effect of not escaping saddle points and their performance can be seriously obstructed by ill-conditioning. Furthermore, SGD methods require exhaustive trial and error to fine-tune many learning parameters. Using second derivative information can result in improved convergence properties, but computing the Hessian matrix for large-scale problems is not practical. Quasi-Newton methods require only first-order gradient information, like SGD, but they can construct a low rank approximation of the Hessian matrix and result in superlinear convergence. The limited-memory Broyden-Fletcher-Goldfarb-Shanno (L-BFGS) approach is one of the most popular quasi-Newton methods that construct positive definite Hessian approximations. In this paper, we introduce an efficient optimization method, based on the limited memory BFGS quasi-Newton method using line search strategy -- as an alternative to SGD methods. Our method bridges the disparity between first order methods and second order methods by continuing to use gradient information to calculate a low-rank Hessian approximations. We provide formal convergence analysis as well as empirical results on a subset of the classic ATARI 2600 games. Our results show a robust convergence with preferred generalization characteristics, as well as fast training time and no need for the experience replaying mechanism.	

\section{Introduction}
\label{sec:intro}
Reinforcement learning (RL) -- a class of machine learning problems -- involves learning how to map situations to actions so as to maximize numerical reward signals received during the experiences that an artificial agent has as it interacts with its environment \citep{RL-Book:Sutton:Barto:1998}.
One of the challenges that arise in real-world RL problems is the ``curse of dimensionality''. Nonlinear function approximators coupled with reinforcement learning have made it possible to learn abstractions over high dimensional state spaces \citep{SuttonRS:1996:Coarse,Rafati-Noelle:2015:CSC,Rafati-Noelle:2017:CCCN,Melo:2008:Analysis-Q-learnings}.
Successful examples of using neural networks for reinforcement learning include learning how to play the game of
Backgammon at the Grand Master
level~\citep{TesauroG:1995:TDGammon}. More recently, researchers at DeepMind Technologies used deep Q-learning algorithm to play various ATARI games from raw screen image stream \citep{DeepMind:Atari:2013,DeepMind:Nature:2015}. The Deep Q-learning algorithm \citep{DeepMind:Atari:2013} employed a convolutional neural network (CNN) as the state-action value function approximation. The resulting performance on these games was frequently at or better than the human expert level. In another effort, DeepMind used deep CNNs and a Monte Carlo Tree Search algorithm that combines supervised learning and reinforcement learning to learn how to play the game of Go at a super-human level \citep{DeepMind-AlphaGo}. 

The majority of deep learning problems, including deep RL algorithms, require solving an unconstrained optimization of a highly nonlinear and nonconvex objective function of the form 
\begin{align}
\min_{w \in \Re^n} \L(w) = \frac{1}{N} \sum_{i}^{N} \ell_i(w)
\label{eq:intro-optimization}
\end{align}
where $w \in \Re^n$ is the vector of trainable parameters of the CNN model. There are various algorithms proposed in machine learning and optimization literature to solve \eqref{eq:intro-optimization}. Among those one can name first-order methods such as stochastic gradient descent (SGD) methods \citep{Robbins:1951:SGD} and quasi-Newton methods \citep{le2011optimization}. For instance, a variant of SGD method was used in DeepMind's implementation of deep Q-Learning algorithm \citep{DeepMind:Nature:2015}. 

Since both $n$ and $N$ are large in large-scale problems, the computation of the true gradient, $\nabla \L(w)$, is expensive and additionally, the computation of the true Hessian, $\nabla^2 \L(w)$, is not practical. At each iteration of learning, SGD methods use a small random sample of data, $J_k$, to compute an approximate of the gradient of the objective function, $\nabla^{(J_k)} \L(w_k)$ and use the opposite of that vector as the search direction, $p_k = - \nabla^{(J_k)} \L(w_k)$. The computational cost-per-iteration of SGD algorithm is small, making them the most widely used optimization method for the vast majority of deep learning and deep RL applications. 

However, these methods require fine-tuning many hyperparameters, including the learning rates $\alpha_k$. The learning rates are usually chosen to be very small to decrease the undesirable effect of the noisy stochastic gradient. Therefore, deep RL methods based on the SGD algorithms require storing a large memory of the recent experiences into a \emph{experience replay memory} $\D$ and replaying this memory repeatedly. Another major drawback of the SGD methods is that they struggle with saddle-points and the problem ill-conditioning that occur in most of the nonconvex optimization and has the undesirable effect on the model's generalization of learning \citep{BotCN18}. 

On the other hand, using second-order curvature information can help with more robust convergence for nonconvex optimization problems \citep{Nocedal-Wright:2006:Numerical-Optimization-Book,BotCN18}. An example of a second-order method is Newton's method where the Hessian matrix, $\nabla^2 \L(w)$ and the gradient are used to find the search direction, $p_k = -\nabla^2 \L(w_k)^{-1} \nabla \L(w_k)$, and then line-search method is used to find the step length along the search direction. The main bottleneck in second-order methods is the serious computational challenges involved in computing the Hessian, $\nabla^2 \L(w)$, for deep reinforcement learning problems, which is not practical when $n$ is large. Quasi-Newton methods and Hessian-free methods both use approaches to approximate the Hessian matrix without computing and storing the true Hessian matrix, $\nabla^2 \L(w)$.  

Quasi-Newton methods form an alternative class of first-order methods for solving the large-scale nonconvex optimization problem in deep learning \citep{Nocedal-Wright:2006:Numerical-Optimization-Book,Erway-etal-2018-arXiv,Rafati-et-al:2018:EUSIPCO}. These methods, like SGD, require only computing the first-order gradient of the objective function. By measuring and storing the difference between consecutive gradients, quasi-Newton methods construct \emph{quasi-Newton matrices}, 
$\{B_k\}$, which are low-rank updates to the previous Hessian approximations for 
estimating $\nabla^2 \L(w_k)$ at each iteration. They build a quadratic model of the objective function by using these quasi-Newton matrices and use that model to find a sequence of search directions that can result in superlinear convergence. Since these methods do not require the second-order derivatives, they are more efficient than Newton's method for large-scale optimization problems \citep{Nocedal-Wright:2006:Numerical-Optimization-Book}. 

There are various quasi-Newton methods proposed in literature. They differ in how they define and construct the quasi-Newton matrices $\{B_k\}$, how the search directions are computed, and how the parameters of the model are updated. The Broyden-Fletcher-Goldfarb-Shanno (BFGS) method  \citep{Bro70,Fle70,Gol70,Sha70} is considered the most popular quasi-Newton algorithm, which produces positive semidefinite matrix $B_k$ for each iteration.

The \emph{Limited-memory} BFGS (L-BFGS) method  constructs a sequence of  low-rank updates to the Hessian  approximation and consequently solving $p_k = B_k^{-1}\nabla \L(w_k) $ 
can be done efficiently. Methods based on L-BFGS quasi-Newton have been implemented and employed for the image classification task in the deep learning framework and impressive results have been produced \citep{Rafati-et-al:2018:EUSIPCO,Rafati-Marcia:2018:ICMLA,Berahas2016multibatch,Le:2011:BFGS-and-CG}.

These methods approximate second derivative information, improving the quality of each training iteration and circumventing the need for application-specific parameter tuning. Given that quasi-Newton methods are efficient in supervised learning problems \citep{BotCN18}, an important question arises: Is it also possible to use  quasi-Newton methods to learn the state representations in deep reinforcement learning successfully? We will investigate this question in the remainder of this paper.

In this paper, we implement a limited-memory BFGS (L-BFGS) optimization method  for deep reinforcement learning framework. Our deep L-BFGS Q-learning method is designed to be efficient for parallel computation in GPU. We experiment our algorithm on a subset of the ATARI 2600 games, by comparing its ability to learn robust representations of the state-action value function, as well as computation and memory efficiency. We also analyze the convergence properties of Q-learning combined with
deep neural network using L-BFGS optimization.

\section{Optimization Problems in RL}
\label{sec:qnrl-rl}
In an RL problem, the agent should implement a policy, $\pi$, from states, $\mathcal{S}$, to possible actions, $\mathcal{A}$, to maximize its expected return from the environment~\citep{RL-Book:Sutton:Barto:2017}. At each cycle of interaction, the agent receives a state, $s$, from the environment, takes an action, $a$, and one time step later, the environment sends a reward, $r \in \mathbb{R}$, and an updated state, $s'$. Each cycle of interaction, $e = (s,a,r,s')$ is called a transition \emph{experience} (or a trajectory). The goal is to find an optimal policy that maximizes the expected value of the return, i.e. the cumulative sum of future rewards, $G_t = \sum_{t'=t}^T \gamma^{t'-t} r_{t'+1}$, where $\gamma \in [0,1]$ is a discount factor, and $T$ as a final step. It is often useful to define a parametrized value function $Q(s,a;w)$ to estimate the expected value of the return. Q-learning is a Temporal Difference (model-free RL) algorithm that attempts to find the optimal value function by minimizing the loss function $L(w)$, which is defined over a recent \emph{experience memory} $\mathcal{D}$: 
\begin{align}
\min_{w \in \mathbb{R}^n} \L(w) \triangleq \frac{1}{2}\mathbb{E}_{e \sim \D} \Big[ \big( \Y - Q(s,a;w) \big)^2\Big],
\label{eq:expected-risk}
\end{align}
where $\Y = r + \max_{a'} Q(s',a';w)$ is the target value for the expected return based on the \emph{Bellman's optimality} equations \citep{RL-Book:Sutton:Barto:1998}. 

In practice, instead of minimization of the expected risk in \eqref{eq:expected-risk} we can define an optimization problem for the \emph{empirical risk} as follows
\begin{align}
\min_{w \in \mathbb{R}^n} \L (w) \triangleq \frac{1}{2|\D|}\sum_{e \in \D} \Big[ \big( \Y - Q(s,a;w) \big)^2\Big].
\label{eq:empirical-risk}
\end{align}
The most common approach for solving the empirical risk minimization problem \eqref{eq:empirical-risk} in literature is using a variant of stochastic gradient decent (SGD) method  \eqref{eq:empirical-risk}. At each optimization step $k$, a small set of experiences $J_k$ are randomly sampled from the \emph{experience replay memory} $\D$. This sample is used to compute an stochastic gradient of the objective function, $\nabla \L (w)^{J_k}$, as an approximate for the true gradient, $\nabla \L (w)$ 
\begin{align}
\nabla \L (w)^{(J_k)} \triangleq \frac{-1}{|J_k|}\sum_{e \in J_k} \Big[ \big( \Y - Q(s,a;w) \big) \nabla Q \Big].
\label{eq:stochastic-gradient}
\end{align}    
The stochastic gradient then can be used to update the iterate $w_k$ to $w_{k+1}$
\begin{align}
w_{k+1} = w_k - \alpha_k \nabla \L (w_k)^{(J_k)},
\label{eq:SGD}
\end{align}    
where $\alpha_k$ is the learning rate (step size).

\section{Line-search L-BFGS Optimization}
\label{sec:quasi-newton}
In this section, we briefly introduce a quasi-Newton optimization method based on the \emph{line-search} strategy, as an alternative for SGD methods. Then we introduce the limited-memory BFGS method. 
\subsection{Line Search Method} 
Each iteration of a line search method computes a search direction $p_k$ and then decides how far to move along that direction. The iteration is given by
\begin{align}
w_{k+1} = w_k + \alpha_k p_k, 
\label{eq:line-search}
\end{align} 
where $\alpha_k$ is called the step size. 
The search direction $p_k$ is obtained by minimizing a quadratic model of the objective function defined by 
\begin{align}
p_k = \min_{p\in \mathbb{R}^n} q_k(p) \triangleq g_k^T p + \frac{1}{2} p^T B_k p, 
\label{eq:quadratic-model2}
\end{align} 
where $g_k = \nabla \L(w_k) \in \mathbb{R}^n$ is the gradient of the objective function at $w_k$, and $B_k$ is an approximation to the Hessian matrix $\nabla^2 \L(w_k)
\in \mathbb{R}^{n \times n}$. 
If $B_k$ is a positive definite matrix, the minimizer of the quadratic function can be found as
\begin{align}
p_k = - B_k^{-1} g_k. 
\label{eq:minimizer-quadratic-model}
\end{align} 
The step size $\alpha_k$ is chosen to satisfy sufficient decrease and curvature conditions, e.g. 
the Wolfe conditions \citep{Nocedal-Wright:2006:Numerical-Optimization-Book} given by
\begin{subequations}
	\begin{align}
	\L(w_k + \alpha_k p_k) & \leq \L(w_k) + c_1 \alpha_k \nabla \L_k^T p_k, \\
	\nabla \L(w_k + \alpha_k p_k)^T p_k &\geq c_2 \nabla \L(w_k)^T p_k,
	\end{align}
	\label{eqn:Wolfe-Conditions2}
\end{subequations}
with $0 < c_1 < c_2 < 1$.
\subsection{Quasi-Newton Optimization Methods} 
Methods that use the Hessian for $B_k$, $B_k = \nabla^2 \L(w_k)$, in the quadratic model in \eqref{eq:quadratic-model2}
typically exhibit quadratic rates of convergence.  However, in large-scale problems (where $n$ and $N$ are both large), computing the true Hessian explicitly is not practical. In this case,
quasi-Newton methods are viable alternatives because they exhibit super-linear convergence rates while maintaining memory and computational efficiency. Instead of the true Hessian, quasi-Newton methods use an approximation $B_k$, which is updated after each step to take account of the additional knowledge gained during the step.

Quasi-Newton methods, like gradient descent methods, require only the computation of first-derivative information. They can construct a model of objective function by measuring the changes in the consecutive gradients for estimating the Hessian. The \emph{quasi-Newton matrices}, $\{B_k\}$, are required to satisfy the secant equation
\begin{align}
B_{k+1}  (w_{k+1}-w_k) \approx \nabla \L(w_{k+1}) - \nabla \L(w_k).
\end{align}
Typically, there are additional conditions imposed on $B_{k+1}$, such as symmetry (since the exact Hessian is symmetric), and a requirement that the update to obtain $B_{k+1}$ from $B_k$ is low rank, meaning that the Hessian approximations cannot change too much from one iteration to the next. Quasi-Newton methods vary in how this update is defined.
\subsection{The BFGS Quasi-Newton Update}
Perhaps, the most well-known among all of the quasi-Newton methods is the
Broyden-Fletcher-Goldfarb-Shanno (BFGS) update \citep{LiuN89,Nocedal-Wright:2006:Numerical-Optimization-Book},
given by 
\begin{align}
B_{k+1}= B_k - \frac{1}{\s_k^T B_k \s_k} B_k \s_k \s_k^T B_k + \frac{1}{\y_k^T \s_k} \y_k \y_k^T,
\label{eqn:bfgs2}
\end{align}
where $\s_k =  w_{k+1} - w_k $ and $\y_k = \nabla \L(w_{k+1}) - \nabla 
\L(w_{k})$. The matrices are defined recursively with the initial matrix, $B_0 = \lambda_{k+1} I$, where the scalar $\lambda_{k+1} > 0$. The BFGS method generates positive-definite approximations whenever the initial approximation $B_0$ is positive definite and $s_k^T y_k > 0$.   
\subsection{Limited-Memory BFGS} 
In practice, only the $m$ most-recently computed pairs $\{(\s_k, \y_k)\}$ are stored, where $m \ll n$, typically $m \le 100$ for very large problems. This approach is often referred to as \emph{limited-memory} BFGS (L-BFGS). Since we have to compute $p_k = - B_k^{-1} g_k$ at each iteration, 
we make use of the following recursive formula for $H_k = B_k^{-1}$:
\begin{align}
H_{k+1} = \Big( I - \frac{\y_k \s_k^T}{\y_k^T \s_k} \Big)H_k \Big(I - \frac{\s_k \y_k^T}{\y_k \s_k^T} \Big) + \frac{\y_k \y_k^T}{\y_k \s_k^T},
\label{eq-L-BFGS-H} 
\end{align}
where $H_0 = \gamma_{k+1} I$, and common value for $\gamma_{k+1}$ is usually chosen to be $\y_k^T \s_k/\y_k^T \y_k$ \citep{Nocedal-Wright:2006:Numerical-Optimization-Book,Rafati-Marcia:2018:ICMLA}. The \emph{L-BFGS two-loop recursion algorithm} given in Algorithm \ref{Algo:L-BFGS-two-loop-recursion} can compute $p_k = -H_k g_k$ in $4mn$ operations \citep{Nocedal-Wright:2006:Numerical-Optimization-Book}. 

\begin{algorithm}
	\begin{algorithmic}
		\State $\q \gets g_k = \nabla \L(w_k)$
		\For{$i=k-1,\dots,k-m$}
		\State $\alpha_i = \frac{\s_i^T q}{\y_i^T \s_i}$
		\State $\q \gets \q - \alpha_i \y_i$
		\EndFor
		\State $\r \gets H_0 q$ 
		\For{$i=k-1,\dots,k-m$}
		\State $\beta = \frac{\y_i^T r}{\y_i^T \s_i}$
		\State $\r \gets \r + \s_i ( \alpha_i - \beta)$
		\EndFor\\
		\Return $- \r = - H_k g_k$
	\end{algorithmic}
	\caption{L-BFGS two-loop recursion.}
	\label{Algo:L-BFGS-two-loop-recursion}
\end{algorithm}

\section{Deep L-BFGS Q Learning Method}
\label{sec:lbfgs-dqn}
In this section, we propose a novel algorithm for the optimization problem in deep Q-Learning framework, based on limited-memory BFGS method within line search strategy. This algorithm is designed to be efficient for parallel computations on a single or multiple GPU(s). Also the experience memory $\D$ is emptied after each gradient computation, hence the algorithm needs much less RAM memory. 

Inspired by \cite{Berahas2016multibatch}, we use the overlap between the consecutive multi-batch samples $O_k = J_k \cap J_{k+1}$ to compute $y_k$  as 
\begin{align}
\y_k = \nabla \L(w_{k+1})^{(O_k)} -  \nabla \L(w_{k})^{(O_k)}.
\label{eq:y-overlap}
\end{align} 
The use of overlap to compute $y_k$ has been shown to result in more robust convergence in L-BFGS since L-BFGS uses  gradient differences to update the Hessian approximations (see \citep{Berahas2016multibatch,Erway-etal-2018-arXiv}).

At each iteration of optimization we collect experiences in $\D$ up to batch size $b$ and use the entire experience memory $\D$ as the overlap of consecutive samples $O_k$. 
For computing the gradient $g_k = \nabla \L(w_k)$, we use the $k$th sample, $J_k = O_{k-1} \cup O_k$
\begin{align}
\nabla \L (w_k)^{(J_k)} = \frac{1}{2}(\nabla \L (w_k)^{(O_{k-1})} + \nabla \L (w_k)^{(O_{k})}).
\label{eq:gradient-Jk}
\end{align}    
Since $\nabla \L(w_k)^{(O_{k-1})}$ is already computed to obtain $\y_{k-1}$ in previous iteration, we only need to compute $\nabla \L^{(O_k)}(w_{k})$, given by 
\begin{align}
\nabla \L (w_k)^{(O_k)} = \frac{-1}{|\D|}\sum_{e \in D} \Big[ \big( \Y - Q(s,a;w_k) \big) \nabla Q \Big].
\label{eq:overlap-gradient}
\end{align}    
Note that in order to obtain $\y_k$, we only need to compute $\nabla \L (w_{k+1})^{(O_{k})}$ since $\nabla \L (w_k)^{(O_k)}$ is already computed when we computed the gradient in \eqref{eq:gradient-Jk}. 

The line search multi-batch L-BFGS optimization algorithm for deep Q-Leaning is provided in Algorithm \ref{Algo:DQN+L-BFGS}.
\begin{algorithm}
	\begin{algorithmic}
		\State \textbf{Inputs:} batch size $b$, L-BFGS memory $m$, exploration rate $\epsilon$ 
		\State \textbf{Initialize} experience memory $\mathcal{D} \gets \emptyset$ with capacity $b$
		\State \textbf{Initialize} $w_0$, i.e. parameters of $Q(.,.;w)$ randomly
		\State \textbf{Initialize} optimization iteration $k \gets 0$
		\For{ episode $=1,\dots,M$} 
		\State Initialize state $s \in \mathcal{S}$
		\Repeat{ for each step $t = 1,\dots,T$} 					
		\State compute $Q(s,a;w_k)$
		\State $a\gets$\texttt{EPS-GREEDY}$(Q(s,a;w_k),\epsilon)$
		\State Take action $a$
		\State Observe next state $s'$ and external reward $r$ 
		\State Store transition experience $e=\{s,a,r,s'\}$ to $\mathcal{D}$
		\State $s \gets s'$
		\Until{$s$ is terminal or intrinsic task is done}
		\If{$|\mathcal{D}| == b$}
		\State $O_k \gets \mathcal{D}$
		\State Update $w_k$ by performing \textbf{optimization step}
		\State $\mathcal{D} \gets \emptyset$
		\EndIf
		\EndFor			
	\end{algorithmic}
	\textbf{========================================}\\
	\textbf{Multi-batch line search L-BFGS Optimization step:}
	\begin{algorithmic}
		\State Compute gradient $g^{(O_k)}_k$
		\State Compute gradient $g^{(J_k)}_k \gets \frac{1}{2} g^{(O_k)}_k + \frac{1}{2} g^{(O_{k-1})}_k$
		\State Compute $p_k = - B_k^{-1} g^{(J_k)}_k$ using Algorithm \ref{Algo:L-BFGS-two-loop-recursion} 
		\State Compute $\alpha_k$ by satisfying the Wolfe Conditions \eqref{eqn:Wolfe-Conditions2} 
		\State Update iterate $w_{k+1} = w_{k} + \alpha_k p_{k}$
		\State $\s_{k} \gets w_{k+1} - w_{k}$
		\State Compute $g^{(O_k)}_{k+1} = \nabla \L(w_{k+1})^{(O_k)}$ 
		\State $\y_{k} \gets g^{(O_k)}_{k+1} - g^{(O_k)}_{k}$
		\State Store $\s_{k}$ to $S_{k}$ and $\y_{k}$ to $Y_k$ and remove  oldest pairs
		\State $k \gets k+1$
		
	\end{algorithmic}
	\caption{Line search Multi-batch L-BFGS Optimization for Deep Q Learning.}
	\label{Algo:DQN+L-BFGS}
\end{algorithm}

\section{Convergence Analysis}
\label{sec:convergence-analysis}
In this section, we present convergence analysis for the deep Q-learning with multi-batch line-search L-BFGS optimization method in Algorithm \ref{Algo:DQN+L-BFGS}. We also provide analysis for optimality of the state action value $Q(.,.;w)$ function. Then we provide comparison between computation time of deep L-BFGS Q-learning (Algorithm \ref{Algo:DQN+L-BFGS}) with DeepMind's Deep Q-learning algorithm \citep{DeepMind:Nature:2015} that uses a variant of SGD method.   
\subsection{Convergence of empirical risk}
To analyze the convergence properties of empirical risk function $\L(w)$ in \eqref{eq:empirical-risk} we assume that
\begin{subequations}
	\begin{align}
	&\L(w) \textrm{ is strongly convex, and twice differentiable}. \label{eq:assumption-1} \\
	&\forall w,~ \exists \lambda,\Lambda>0 \textrm{ such that } \lambda I \preceq \nabla ^2\L(w) \L \preceq \Lambda I, \textrm{ i.e. Hessian is bounded.} \label{eq:assumption-2}\\
	&\forall w,~\exists \eta > 0 \textrm{ such that } \| \nabla \L(w) \|^2 \leq \eta^2, \textrm{ i.e. Gradient does not explode.} \label{eq:assumption-3}
	\end{align}
	\label{eq:assumptions}
\end{subequations}

\begin{lemma}
	$\exists \lambda',\Lambda'>0$ such that $\lambda' I \preceq H_k \preceq \Lambda' I$.
	\label{lemma:1}
\end{lemma}

\begin{proof}
	Due to the assumptions \eqref{eq:assumption-1} and \eqref{eq:assumption-2}, the eigenvalue of positive-definite matrix $H_k$ are also bounded     \citep{byrd2016stochastic,Berahas2016multibatch}.   
\end{proof}

\begin{lemma}
	Let $w^*$ be minimizer of $\L$, then for all $w$, we have $2 \lambda (\L(w) - \L(w^*) \leq \| \nabla \L(w) \|^2$. 
	\label{lemma:2}
\end{lemma}

\begin{proof}
	For any convex function $\L$, and for any two points $w,w^*$, one can show that \citep{Nesterov:2013} 
	\begin{align}
	\begin{split}
	\L(w) \leq \L(w^*) + \nabla \L(w^*)^T (w-w^*) \\+ \frac{1}{2\lambda}\| \nabla \L(w) - \nabla \L(w^*)\|^2.
	\end{split}
	\label{eq:nestrov}
	\end{align}
	Since $w^*$ is minimizer of $\L$ then $\nabla \L(w^*)=0$ in \eqref{eq:nestrov} and we have the proof.
\end{proof}

\begin{theorem}
	Let $w_k$ be iterates generated by Algorithm \ref{Algo:DQN+L-BFGS}, and let's assume that the step length $\alpha_k$ is fixed. The upper bound for the  empirical risk offset from the true minimum value is 
	\begin{align}
	\begin{split}
	\L(w_k) - \L(w^*) \leq (1 - 2 \alpha \lambda \lambda' )^k [\L(w_0) - \L(w^*)] \\ 
	+ [1 - (1 - 2 \alpha \lambda \lambda')^k]\frac{\alpha^2 \Lambda'^2 \Lambda \eta^2}{4 \lambda' \lambda}
	\end{split}
	\label{eq:L-bound}
	\end{align}  
\end{theorem}

\begin{proof}
	By using Taylor expansion on 
	\[
	\L(w_{k+1}) = \L(w_k - \alpha_k H \nabla \L(w_k)) 
	\]
	around $w_k$ we can have
	\begin{align}
	\begin{split}
	\L(w_{k+1}) \leq \L(w_{k}) - \alpha_k \nabla \L(w_k)^T H_k \nabla \L(w_k) \\+ \frac{\Lambda}{2} \|  \alpha_k \nabla \L(w_k)^T H_k \nabla \L(w_k) \|^2.
	\end{split}
	\end{align}
	By applying assumptions \eqref{eq:assumptions} and Lemma \ref{lemma:1} and \ref{lemma:2} to above expression, we have
	\begin{align}
	\begin{split}
	&\L(w_{k+1}) \leq \L(w_{k}) \\&- 2 \alpha_k \lambda' \lambda [\L(w_k) - \L(w^*)] + \frac{\alpha_k^2 \Lambda'^2 \Lambda \eta^2}{4 \lambda' \lambda}
	\end{split}
	\end{align}
	By rearranging above expression and recursion over $k$ we have the proof. For more detailed proof see \cite{byrd2016stochastic,Berahas2016multibatch}.    
\end{proof}
If the step size is bounded $\alpha\in(0,1/2\lambda\lambda')$, we can conclude that the first term of the bound given in \eqref{eq:L-bound} is decaying linearly to zero when $k \to \infty$ and the constant residual term $\frac{\alpha^2 \Lambda'^2 \Lambda \eta^2}{4 \lambda' \lambda}$ is the neighborhood of convergence.

\subsection{Value Optimality}
The Q-learning method is proved to converge to the optimal value function if the step sizes satisfies $\sum_{k}\alpha_k = \infty$ and $\sum_{k}\alpha_k^2 < \infty$ \citep{Jaakola:1994:q-learning-convergence}. 
Now we want to prove that the Q-learning using the L-BFGS update also theoretically converges to the optimal value function under extra condition on the step length $\alpha_k$.
\begin{theorem}
	Let $Q^*$ be the optimal state-action value function and $Q_{k}$ is the Q-function with parameters $w_{k}$. Furthermore, assume that the gradient of $Q$ is bounded $\| \nabla Q \|^2 \leq \eta''^2$ and Hessian of $Q$ functions satisfy $\lambda''\preceq \nabla^2 Q \preceq \Lambda''$. We have
	\begin{align}
	\begin{split}
	&\| Q_{k+1} -  Q^* \|_{\infty} < \\ & \prod_{j=0}^k \big[1 - \alpha_j \eta''^2 \lambda + \frac{\alpha_j \eta''^2 \Lambda'^2 \Lambda''}{2}\big]^{k} \| Q_{0} -  Q^* \|_{\infty}.
	\end{split}
	\label{eq:value-bound}
	\end{align}
	If step size $\alpha_k$ satisfies
	\begin{align}
	\big|1 - \alpha_k \eta''^2 \lambda + \frac{\alpha_k \eta\eta' \Lambda'^2 \Lambda''}{2}\big|<1,~\forall k,
	\label{eq:alpha-cond-value-bound}
	\end{align}
	$Q(.,.;w_k)$ ultimately will converge to $Q^*$, when $k \to \infty$. 
	\label{theorem:value-optimality}
\end{theorem}  

\begin{proof}
	First we derive the effect of parameter update from $w_k$ to 
	\[
	w_{k+1} = w_k - \alpha_k H_k \nabla \L (w_k)
	\]
	on the optimality neighbor.       
	\begin{align}
	\| Q_{k+1} -  Q^* \|_{\infty} \triangleq \max_{s,a} \big| Q(s,a,w_{k+1}) - Q^*(s,a) \big|
	\end{align}
	We approximate the gradient using only one experience $(s,a,r,s')$,
	\begin{align}
	\nabla \L(w_k) \approx \big(Q(s,a;w_k) - Q^*(s,a;w_k) \big)\nabla Q_k(s,a;w_k),
	\label{eq:one-ex-grad}
	\end{align}
	We use the Taylor's expansion to approximate $Q(s,a,w_{k+1})$ 
	\begin{align}
	\begin{split}
	&Q(s,a;w_{k+1}) = Q(s,a;w_{k} - \alpha_k H_k \nabla \L(w_k)) \\
	&= Q(s,a;w_k) -  \alpha_k  \nabla \L_k^T H_k \nabla Q_k + \frac{\alpha_k^2}{2}  \nabla \L_k^T H_k \nabla^2 Q(\xi_k) H_k \nabla \L_k^T \\
	&= Q_k -  \alpha_k  (Q_k-Q^*) \nabla Q_k^T H_k \nabla Q_k + \frac{\alpha_k^2}{2} (Q_k-Q^*) \nabla Q_k^T H_k \nabla^2 Q(\xi_k) H_k \nabla \L_k^T,
	\end{split}
	\end{align}
	where $w_k<\xi_k<w_{k+1}$, $Q_k \coloneqq Q(s,a;w_k)$, $\nabla Q_k \coloneqq \nabla Q(s,a;w_k)$, and $\nabla \L_k \coloneqq \nabla \L(w_k)$. We can use the above expression to compute $\| Q_{k+1} -  Q^* \|_{\infty}$
	\begin{align}
	\begin{split}
	&\| Q_{k+1} -  Q^* \|_{\infty}= \\ 
	&\max_{s,a} \Big| (Q_k-Q^*)  {\Big[1 -\alpha_k \nabla Q_k^T H_k \nabla Q_k + \frac{\alpha_k^2}{2} \nabla Q_k^T H_k \nabla^2 Q(\xi_k) H_k \nabla \L_k\Big] \Big|_{(s,a)}}.
	\end{split}
	\label{eq:q-optimal-2}
	\end{align}
	If $\alpha_k$ satisfies
	\begin{align}
	\Big|1 -\alpha_k \nabla Q_k^T H_k \nabla Q_k +\frac{\alpha_k^2}{2} \nabla Q_k^T H_k \nabla^2 Q_k H_k \nabla \L_k\Big| < 1,
	\label{eq:cond-alpha-conv}
	\end{align}
	then 
	\begin{align}
	\| Q_{k+1} -  Q^* \|_{\infty} < \| Q_{k} -  Q^* \|_{\infty}. 
	\label{eq:q-optimal-3}
	\end{align} 
	Therefore, $Q_k$ converges to $Q^*$ when $k \to \infty$. Considering our assumptions on the bounds of the eigenvalues of $\nabla^2 Q_k$ and $H_k$, we can derive \eqref{eq:alpha-cond-value-bound} from \eqref{eq:cond-alpha-conv}. Recursion on \eqref{eq:q-optimal-2} from $k=0$ to $k+1$ results in \eqref{eq:value-bound}.  
\end{proof}     
\subsection{Computation time}
Let us compare the cost of deep L-BFGS Q-learning in Algorithm \ref{Algo:DQN+L-BFGS} with DQN algorithm in \citep{DeepMind:Nature:2015} that uses a variant of SGD. Assume that the cost of computing gradient is $\O(bn)$ where $b$ is the batch size. The real cost is probably less than this due to the parallel computation in GPU. Let's assume that we run both algorithm for $L$ steps. We update the weights with frequency of every $b$ steps, hence there is $L/b$ maximum updates in our algorithm. The SGD batch size in \citep{DeepMind:Nature:2015} $b_s$ is smaller than $b$ but the frequency of the update is high $f \ll b$. Each iteration of L-BFGS algorithm update consists of the cost of computing the gradient $g_k^{(O_k)}$ which is $bn$, cost of computing the search step $p_k = -H_k g_k^{(O_k)}$ using L-BFGS two-loop recursion (Algorithm \ref{Algo:L-BFGS-two-loop-recursion}) which is $4mn$, and the cost of satisfying the Wolfe conditions \eqref{eqn:Wolfe-Conditions2}, to find step size where most of the times automatically satisfies for $\alpha=1$ and in some steps require recomputing gradient for $z$ times. Therefore we have
\begin{align}
\begin{split}
\frac{\textrm{Cost of Algorithm \ref{Algo:L-BFGS-two-loop-recursion}}}{\textrm{Cost of DQN \citep{DeepMind:Nature:2015}}} &= \frac{(L/b)(zbn + 4mn)}{(L/f)(b_s n)}\\
&= \frac{fz}{b_s} + \frac{4fm}{b b_s}.
\end{split}
\end{align}
In our algorithm, we use quite large batch size to compute less noisy gradients. With $b = 2048$, $b_s = 32$, $f = 4$, $z=5$, $m=20$, the runtime cost ratio will be around $0.63 < 1$. Although per-iteration cost of SGD algorithm is lower than L-BFGS, but the total training time of our algorithm is less than DQN \citep{DeepMind:Nature:2015} for same number of RL steps due to less frequent updates in L-BFGS method.                

\section{Experiments on ATARI 2600 Games}
\label{sec:experiment}
We performed experiments using Algorithm \ref{Algo:DQN+L-BFGS} on six ATARI 2600 games -- Beam Rider, Breakout, Enduro, Q*bert, Seaquest, and Space Invaders. We used OpenAI's gym ATARI environments \citep{OpenAI} which is a wrapper on Arcade Learning Environment emulator \citep{Bellemare:2013:ALE}. These games have been used by other researchers with different learning methods \citep{Bellemare:2012:Contingency,Bellemare:2013:ALE,Hausknecht:2014:HNeat-ATARI,DeepMind:Nature:2015,Schulman:2015:TRPO-ATARI}, and hence they serve as benchmark environments for evaluation of deep reinforcement learning algorithms. 

We used the DeepMind's Deep Q-Network (DQN) architecture in \cite{DeepMind:Nature:2015} as a function approximator for $Q(s,a;w)$. The same architecture was used to train the different ATARI games. The raw Atari frames, which are $210 \times 160$ pixel images with a 128 color palette are preprocessed by first converting their RGB representation to gray-scale and then down-sampling it to a $110\times84$ image. The final input representation is obtained by cropping an $84 \times 84$ region of the image that roughly captures the playing area. The stack of the last 4 consecutive frames was used to produce the input of size  $(4 \times 84 \times 84)$ to the $Q$-function. The first hidden layer of the network consists of a $32$ convolutional filters of size $8 \times 8$ with stride $4$, followed by a Rectified Linear Unit (ReLU) for nonlinearity. The second hidden layer consists of $64$ convolutional filters of size $4 \times 4$ with stride 2, followed by a ReLU function. The third layer consists of 512 of fully-connected linear units followed by ReLU. The output layer is a fully-connected linear layer with a output, $Q(s,a_i,w)$, for each valid joystick action, $a_i \in \mathcal{A}$.

We only used $2000\times 1024$ training steps for training the network on each game (instead of $50$ million steps that was used originally in \cite{DeepMind:Nature:2015}). The training was stopped if the norm of gradient, $\| g_k \|$, was less than a threshold. We used $\epsilon$-greedy for exploration strategy, and similar to  \cite{DeepMind:Nature:2015}, the exploration rare, $\epsilon$, annealed linearly from $1$ to $0.1$. 

Every 10,000 steps, the performance of the learning algorithm was tested by freezing the Q-network's parameters. During the test time, We used $\epsilon=0.05$. The greedy action, $\max_a Q(s,a;w)$, was chosen by the Q-network $95\%$ of the times and there was $5\%$ randomness similar to the DeepMind's implementation in \cite{DeepMind:Nature:2015}.      

Inspired by \cite{DeepMind:Nature:2015}, we also used separate networks to compute the target values, $\Y = r + \gamma \max_{a'} Q(s',a',w_{k-1})$, which was essentially the network with parameters in previous iterate. After each iteration of the multi-batch line search L-BFGS, $w_k$ was updated to $w_{k+1}$, and the target network's parameter $w_{k-1}$ was updated to $w_k$.

Our optimization method was different than DeepMind's RMSProp method used in \cite{DeepMind:Nature:2015} (which is a variant of SGD). We used  stochastic line search L-BFGS method as the optimization method (Algorithm \ref{Algo:DQN+L-BFGS}). There are few important differences between our implementation of deep reinforcement learning method in comparison to the DeepMind's DQN algorithm in \cite{DeepMind:Nature:2015}. 

We used a quite large batch size $b$ in comparison to \cite{DeepMind:Nature:2015}. We experimented our algorithm with different batch sizes $b\in$ \{512, 1024, 2048, 4096, 8192\}. The experience memory $\D$ had a capacity of $b$ also. We used one NVIDIA Tesla K40 GPU with 12GB GDDR5 RAM. The entire experience memory $\D$ could fit in the GPU RAM with a batch size of $b \leq 8192$. 

After every $b$ steps of interaction with the environment, the optimization step in Algorithm \ref{Algo:DQN+L-BFGS} was ran. We used the entire experience memory, $\D$, for the overlap, $O_k$, between two consecutive samples, $J_k$ and $J_{k+1}$, to compute the gradient in \eqref{eq:overlap-gradient} as well as $\y_k$ in \eqref{eq:y-overlap}. Although the DeepMind's DQN algorithm in \cite{DeepMind:Nature:2015} is using smaller batch size of $32$, but the frequency of optimization step is high (every $4$ steps). We hypothesize that using the smaller batch size make the computation of the gradient too noisy, and also doesn't save significant computational time, since the overhead of data transfer between GPU and CPU is more costly than the computation of the gradient on a bigger batch size, due to the power of parallelism in GPU. Once the overlap gradient, $g_k^{(O_k)}$, was computed, we computed the gradient $g_k^{(J_k)}$ for the current sample, $J_k$, in \eqref{eq:gradient-Jk} by memorizing and using the gradient information from the previous optimization step. Then the L-BFGS two loop-recursion in Algorithm \ref{Algo:L-BFGS-two-loop-recursion} was used to compute the search direction $p_k = - H_k g_k^{(J_k)}$.

After finding the quasi-Newton decent direction, $p_k$, the Wolfe Condition \eqref{eqn:Wolfe-Conditions2} was applied to compute the step size, $\alpha_k \in [0.1,1]$, by satisfying the sufficient decrease and the curvature conditions \citep{Wolfe1969,Nocedal-Wright:2006:Numerical-Optimization-Book}. In most of the optimization steps, either the step size of $\alpha_k=1$ satisfied the Wolfe conditions in \eqref{eqn:Wolfe-Conditions2}, or the line search algorithm iteratively used smaller $\alpha_k$ until it satisfied the Wolfe conditions or reached to a lower bound of $0.1$. The original DQN algorithm used a small fixed learning rate of $0.00025$ to avoid the execrable drawback of the noisy stochastic gradient decent step which makes the learning to be very slow.  

The vectors $\s_k = w_{k+1} - w_k$ and $\y_k = g^{(O_k)}_{k+1} - g^{(O_k)}_{k}$ was only added to the recent collections $S_k$ and $Y_k$ only if $\s_k^T \y_k > 0$ and not close to zero. We applied this condition  to \emph{cautiously} preserve the positive definiteness of the L-BFGS matrices $B_k$. Only the $m$ recent $\{(\s_i,\y_i)\}$ pairs were stored into $S_k$ and $Y_k$ ($|S_k| = m$ and $|Y_k|=m$) and the older pairs were removed from the collections. We experimented our algorithm with different L-BFGS memory sizes $m \in \{20,40,80\}$. 

All code is implemented in Python language using Pytorch, NumPy and SciPy libraries and is available at \url{http://rafati.net/quasi-newton-rl}.

\section{Results and Discussions}
\label{sec:discussions}
The average of the maximum game scores is reported in Figure \ref{fig:train-mean-std} (a). The error bar in Figure \ref{fig:train-mean-std} (a) is the standard deviation for the simulations with different batch size, $b \in \{512,1024,2048,4096\}$, and different L-BFGS memory size, $m \in \{20,40,80\}$, for each ATARI game (total of 12 simulations per each task). All simulations regardless of the batch size, $b$, and the memory size, $m$, performed a robust learning. The average training time for each task along with the Squared Temporal Difference (STD) error is shown in Figure \ref{fig:train-mean-std} (b). We did not find a correlation between the training time versus the different batch size, $b$, or the different L-BFGS memory size, $m$. In most of the simulations, the STD error for the training time as shown in Figure \ref{fig:train-mean-std} (b) was not significant.

\begin{figure}[hbt!]
	\centering
	\begin{tabular}[t]{cc}
		(a)&(b)\\
		\includegraphics[width=.44\textwidth]{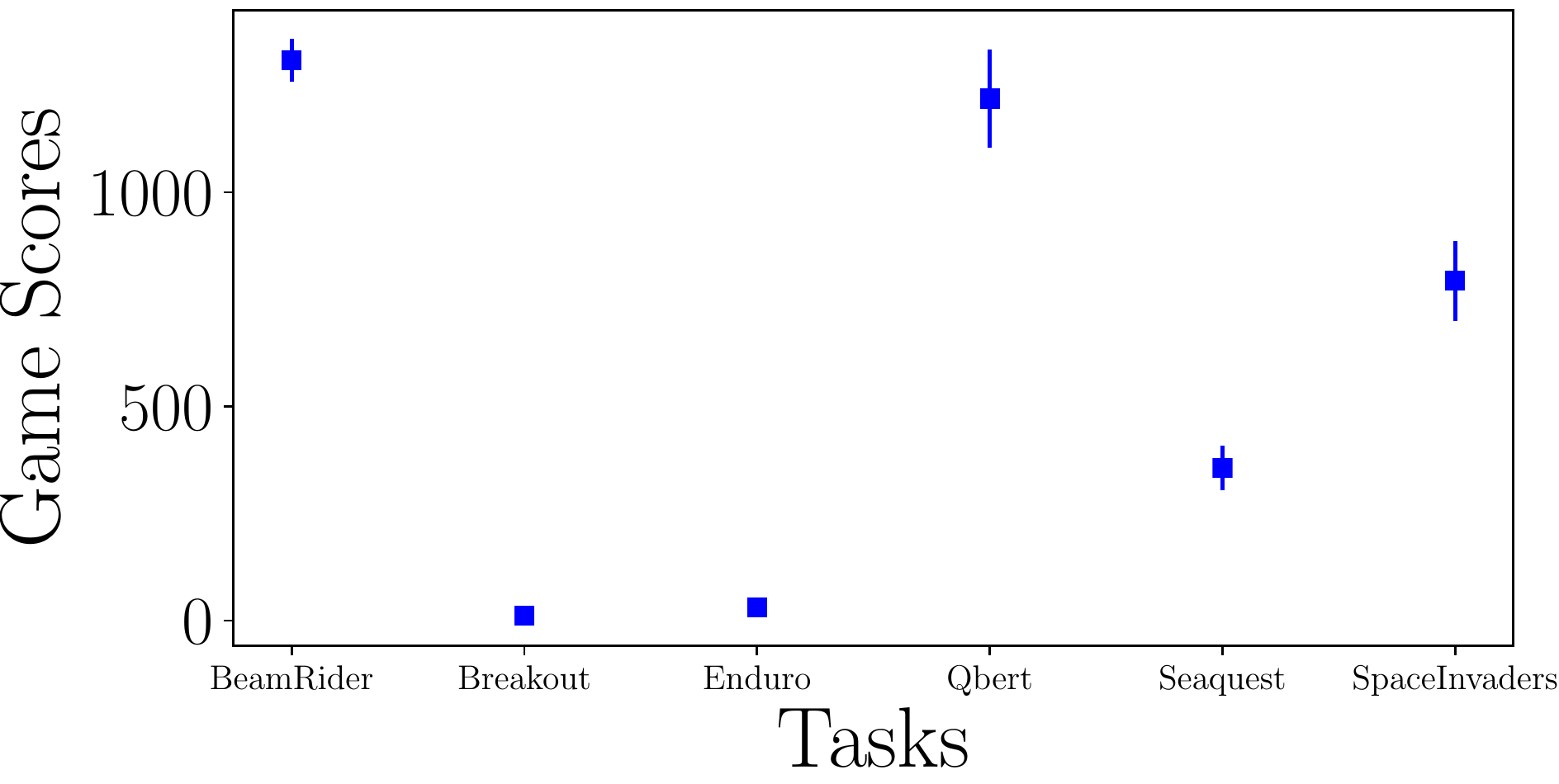}
		&
		\includegraphics[width=.44\textwidth]{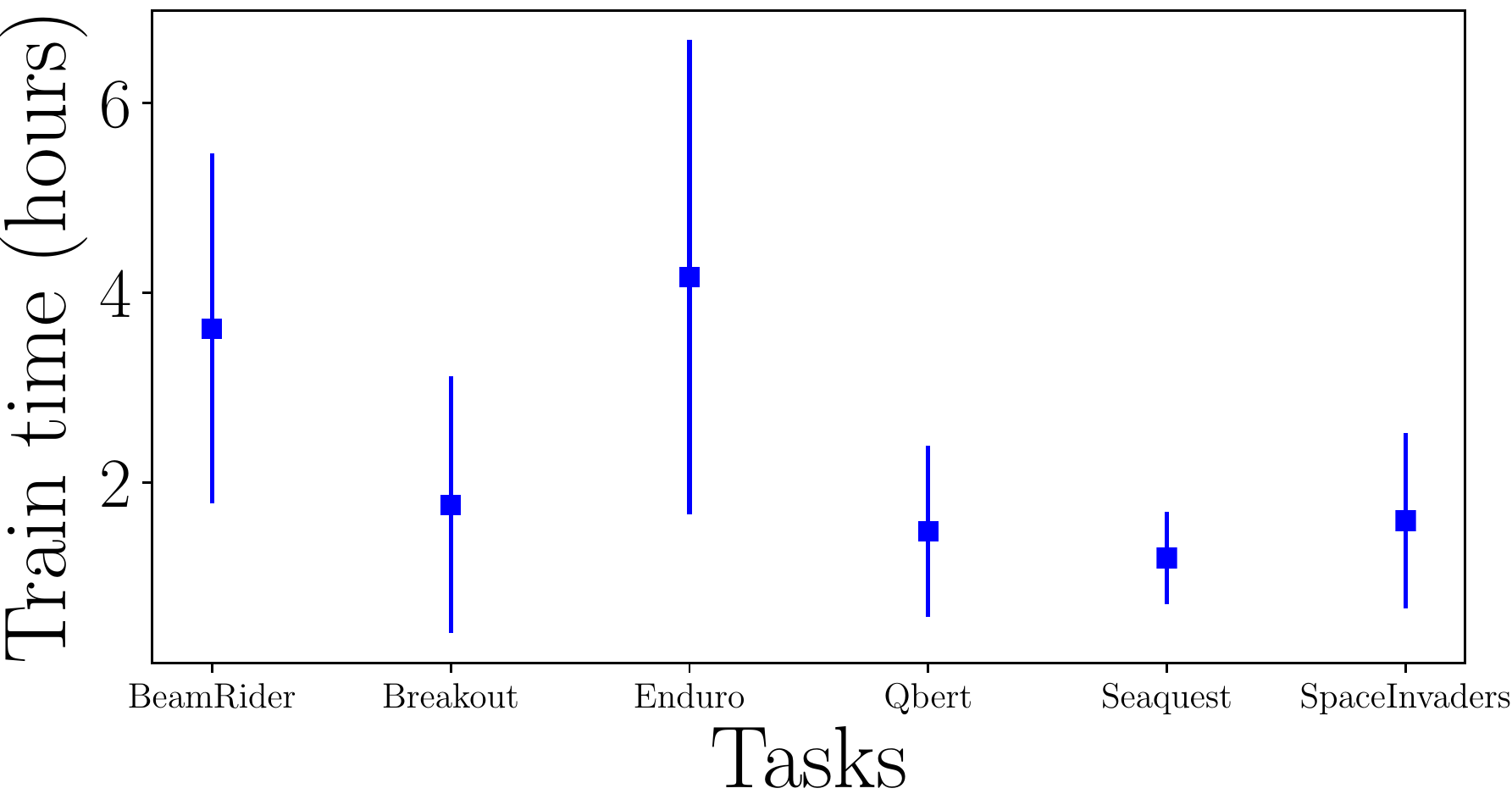}	
	\end{tabular}
	\caption{(a) Test scores (b) Total training time for ATARI games.}
	\label{fig:train-mean-std}
\end{figure}

The test scores and the training loss, $\L_k$, for the six ATARI 2600 environments is shown in Figure \ref{fig:score-time-ATARI-Games} using the batch size of $b=2048$ and L-BFGS memory size $m=40$.

\begin{figure*}[hbt!]
	\centering
	\begin{tabular}{ccc} 
		(a) & (b) & (c) \\
		\includegraphics[width=.3\textwidth]{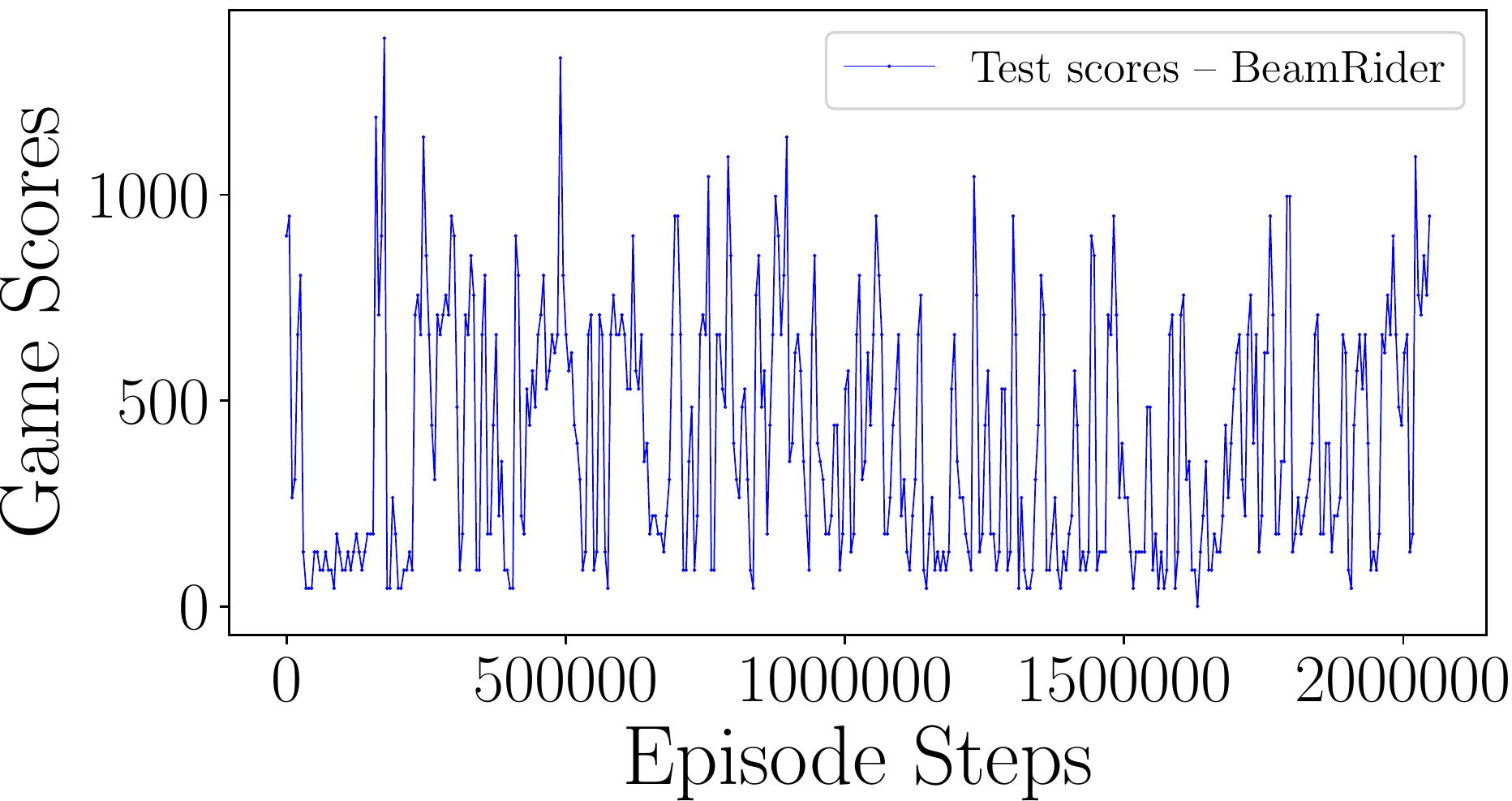} & 
		\includegraphics[width=.3\textwidth]{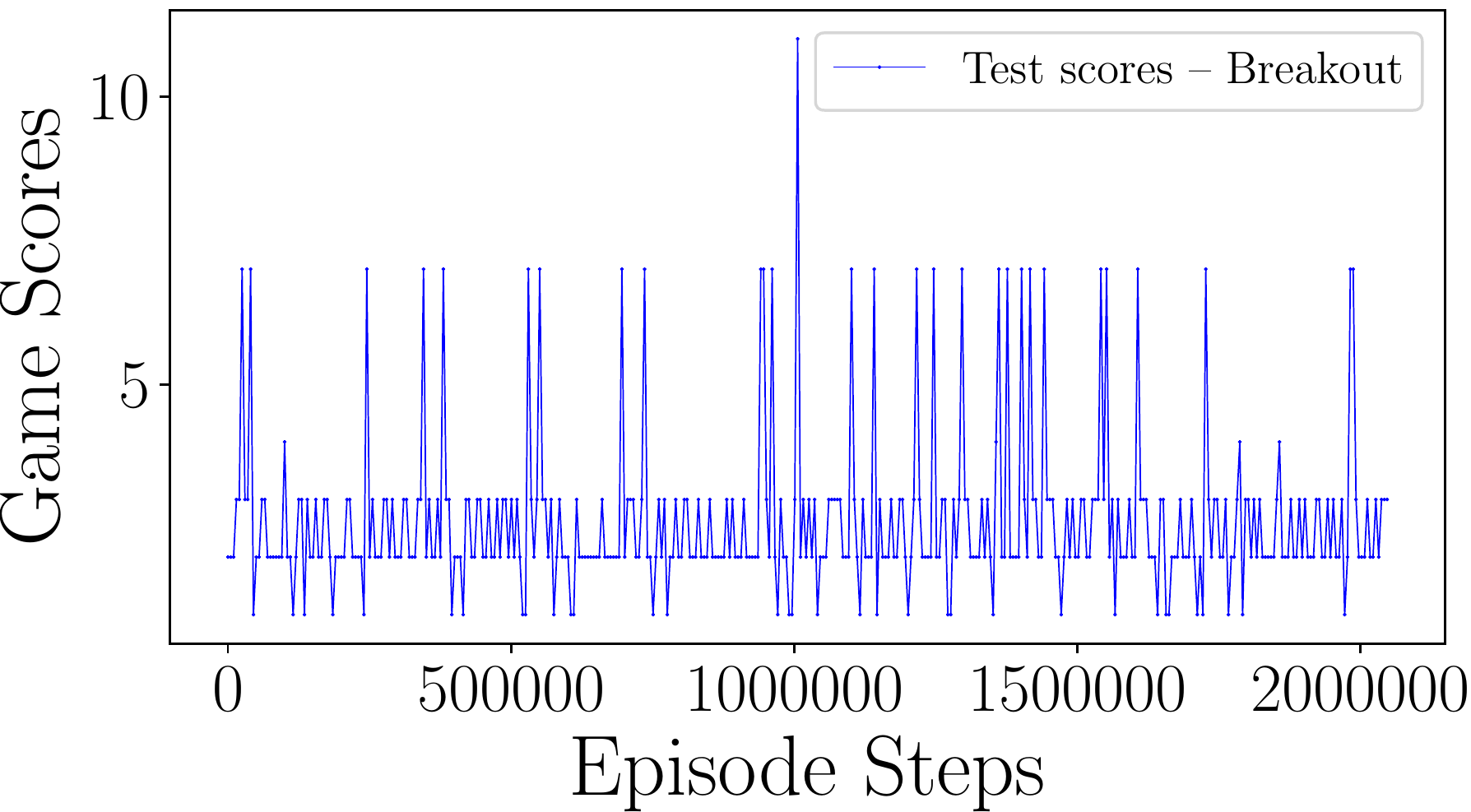} & 	\includegraphics[width=.3\textwidth]{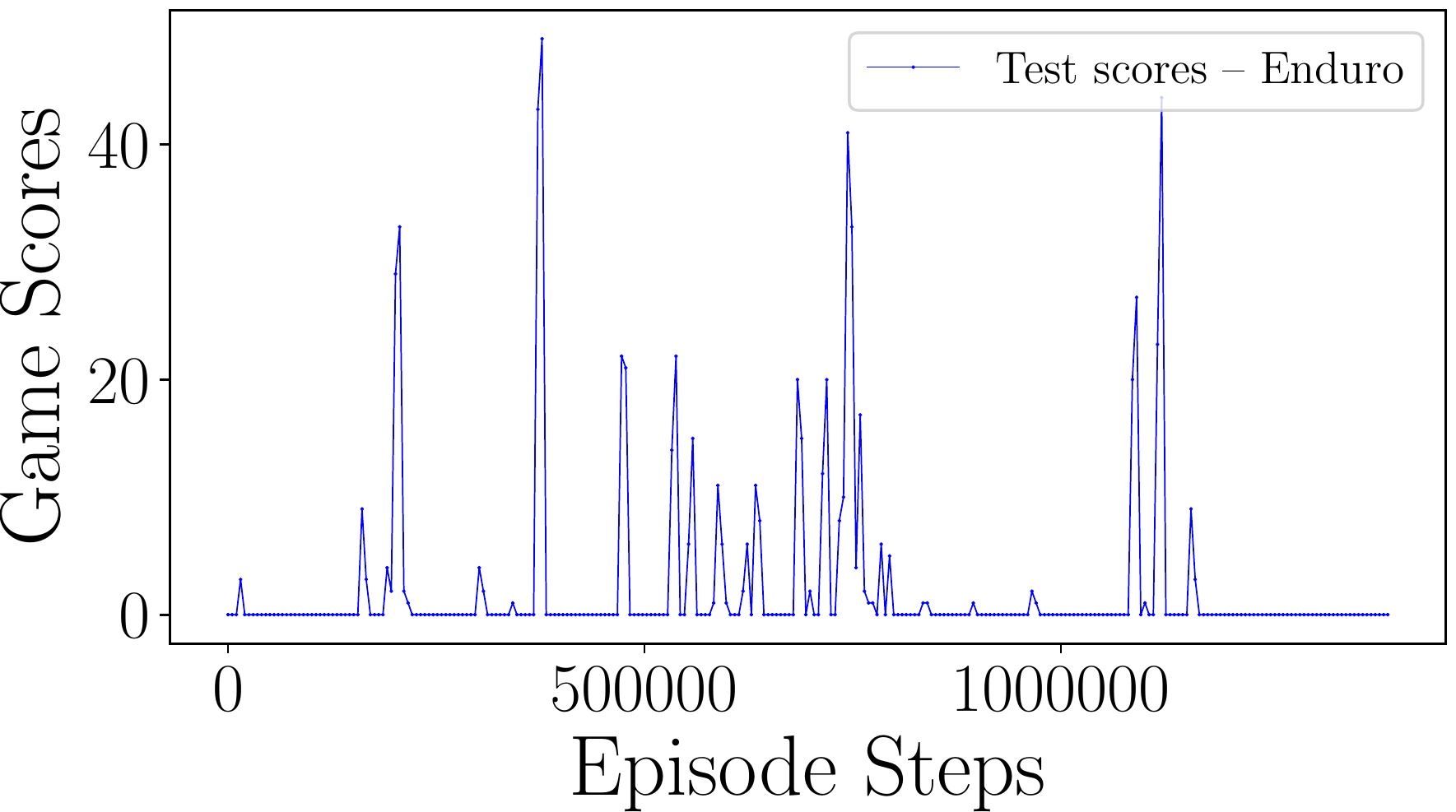} \\
		(d) & (e) & (f) \\
		\includegraphics[width=.3\textwidth]{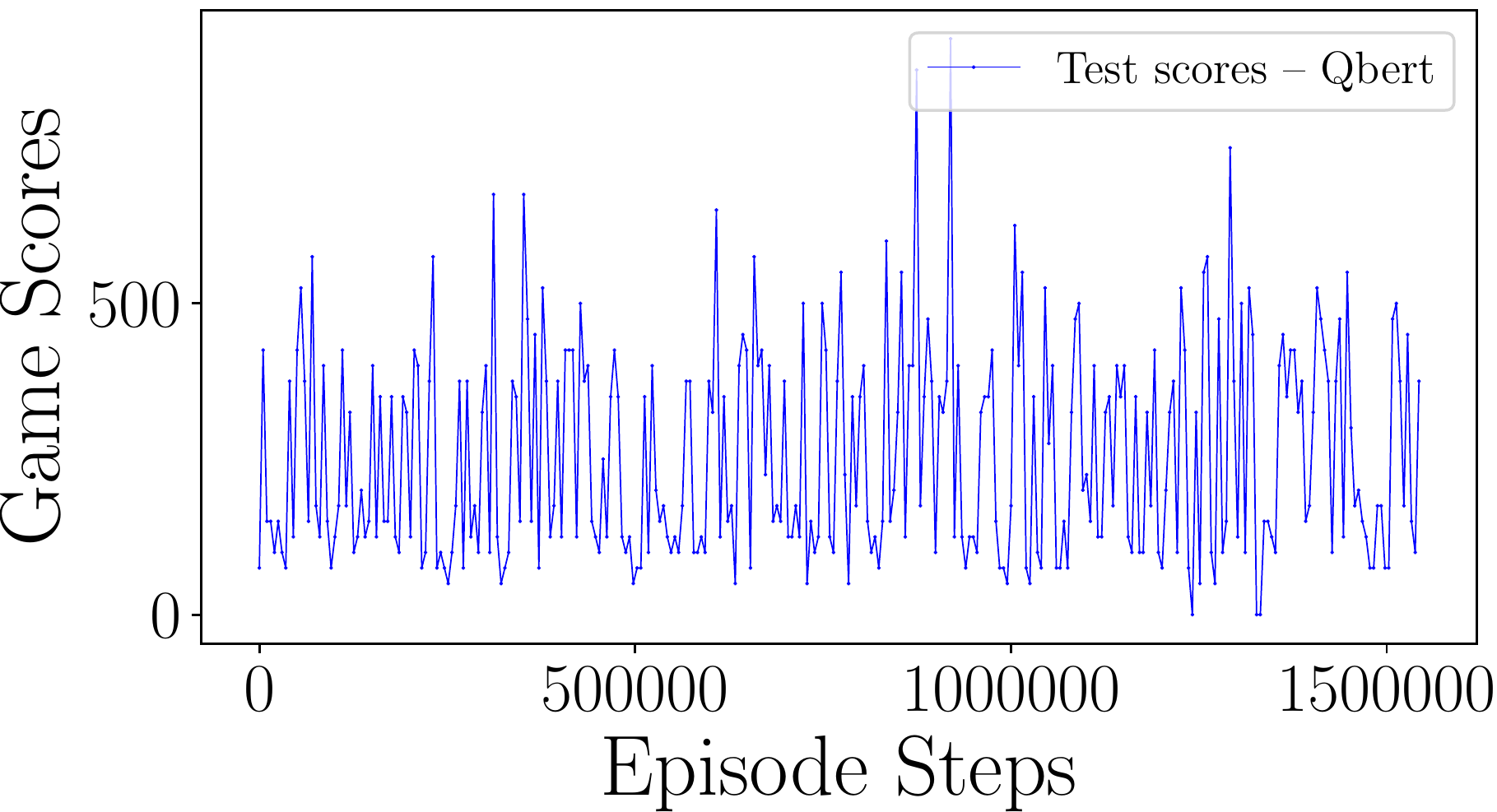} & 
		\includegraphics[width=.3\textwidth]{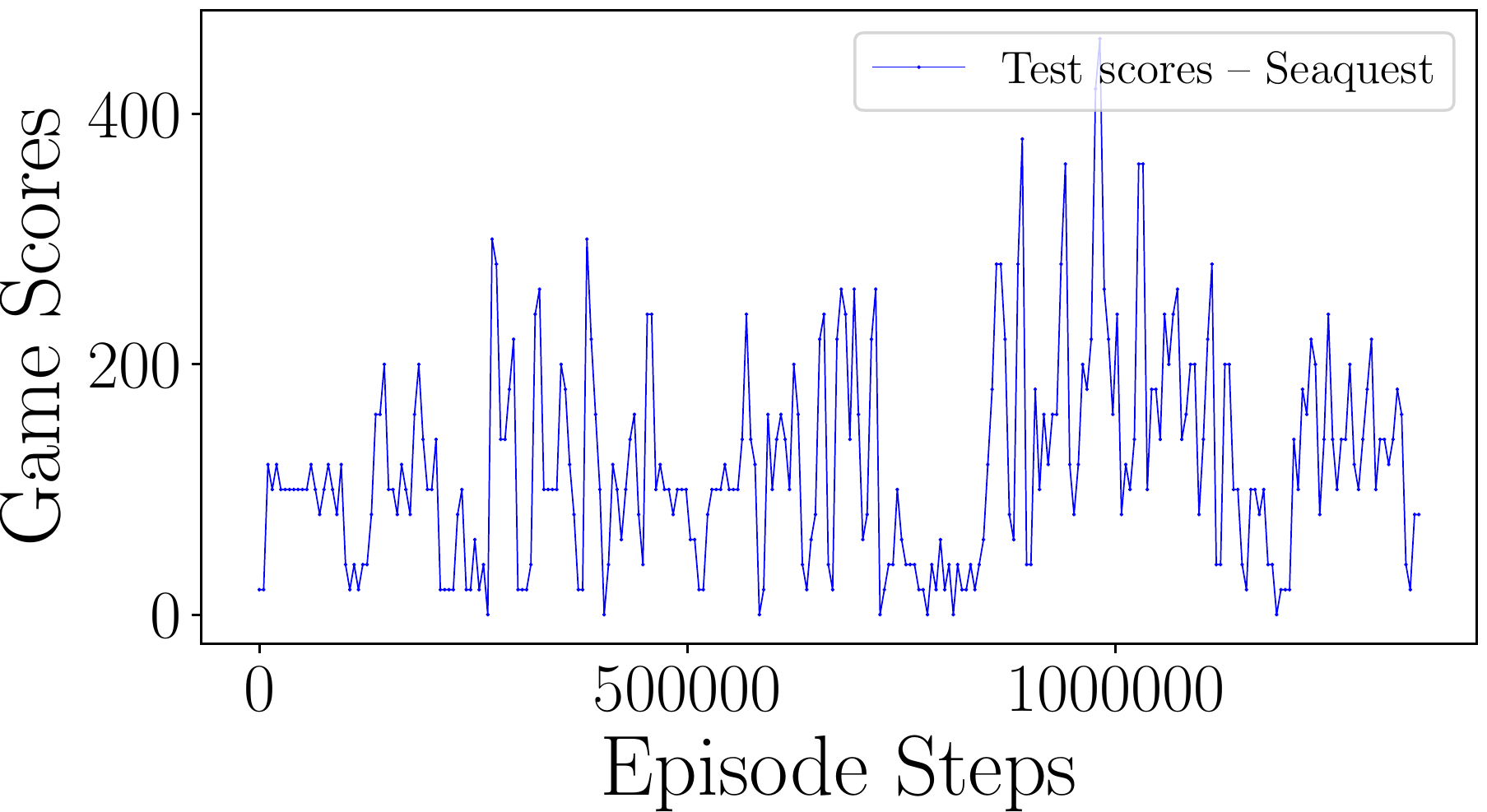} & 	\includegraphics[width=.3\textwidth]{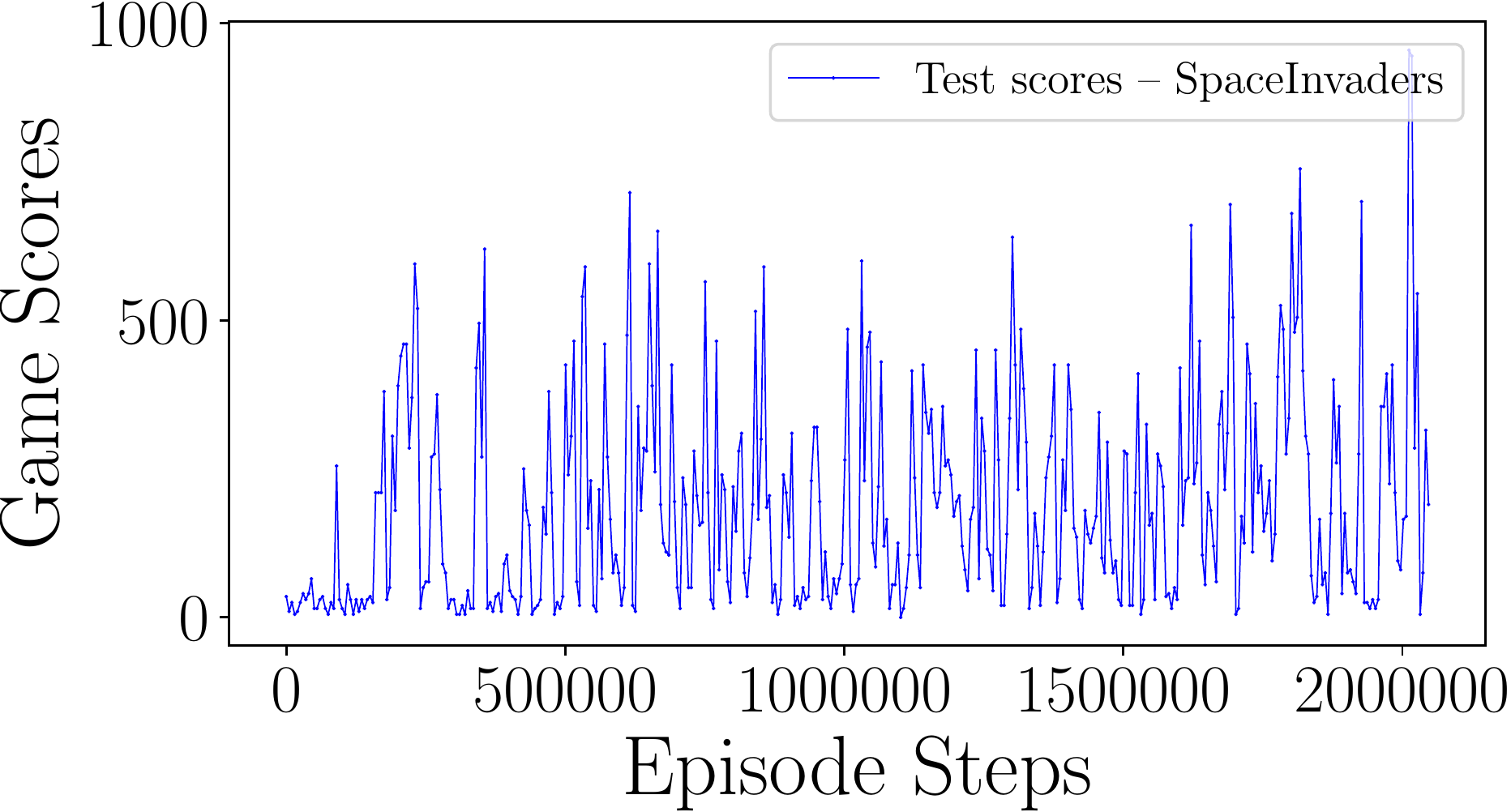} \\
		(g) & (h) & (i) \\ 
		\includegraphics[width=.3\textwidth]{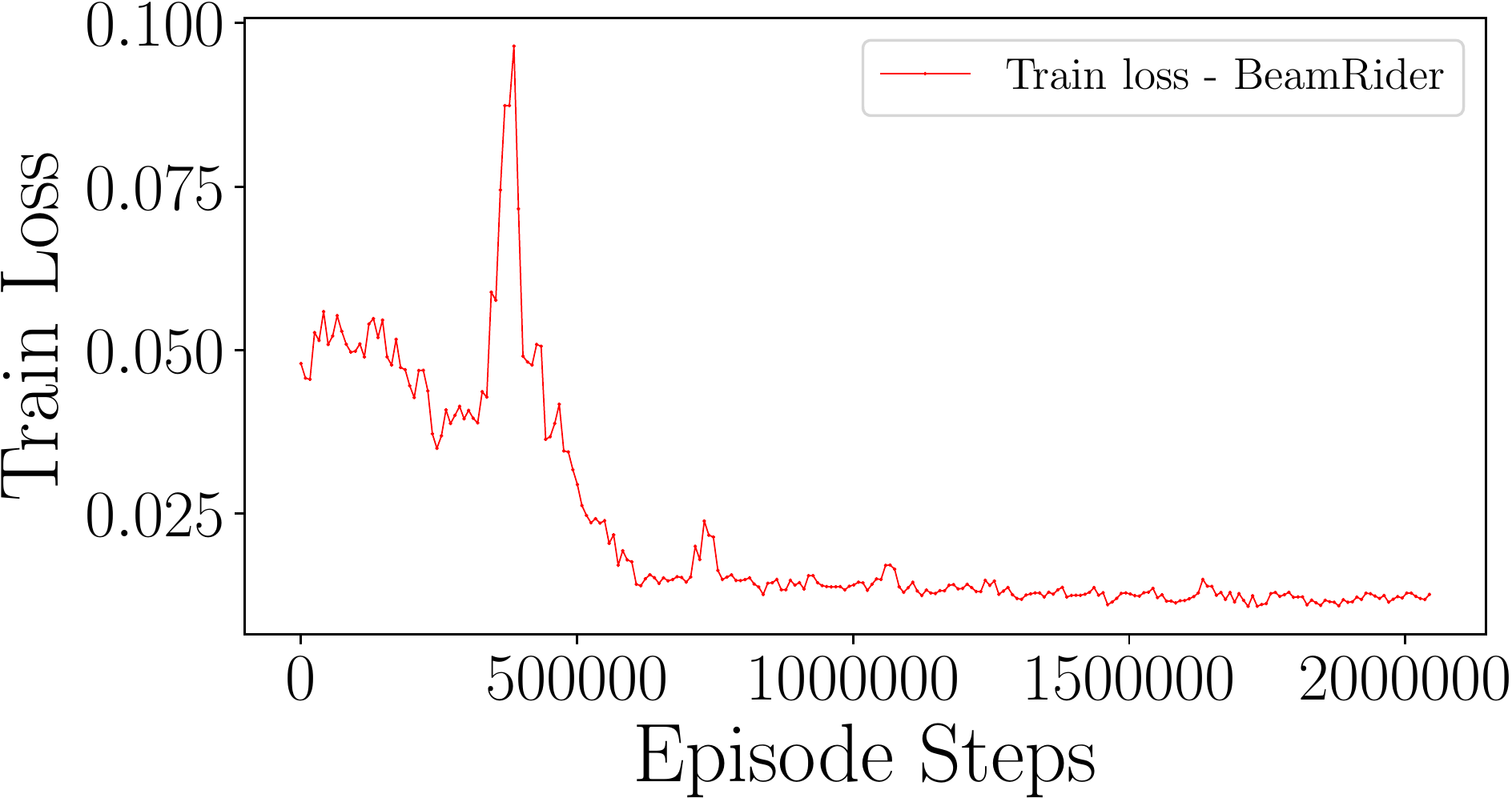} & 
		\includegraphics[width=.3\textwidth]{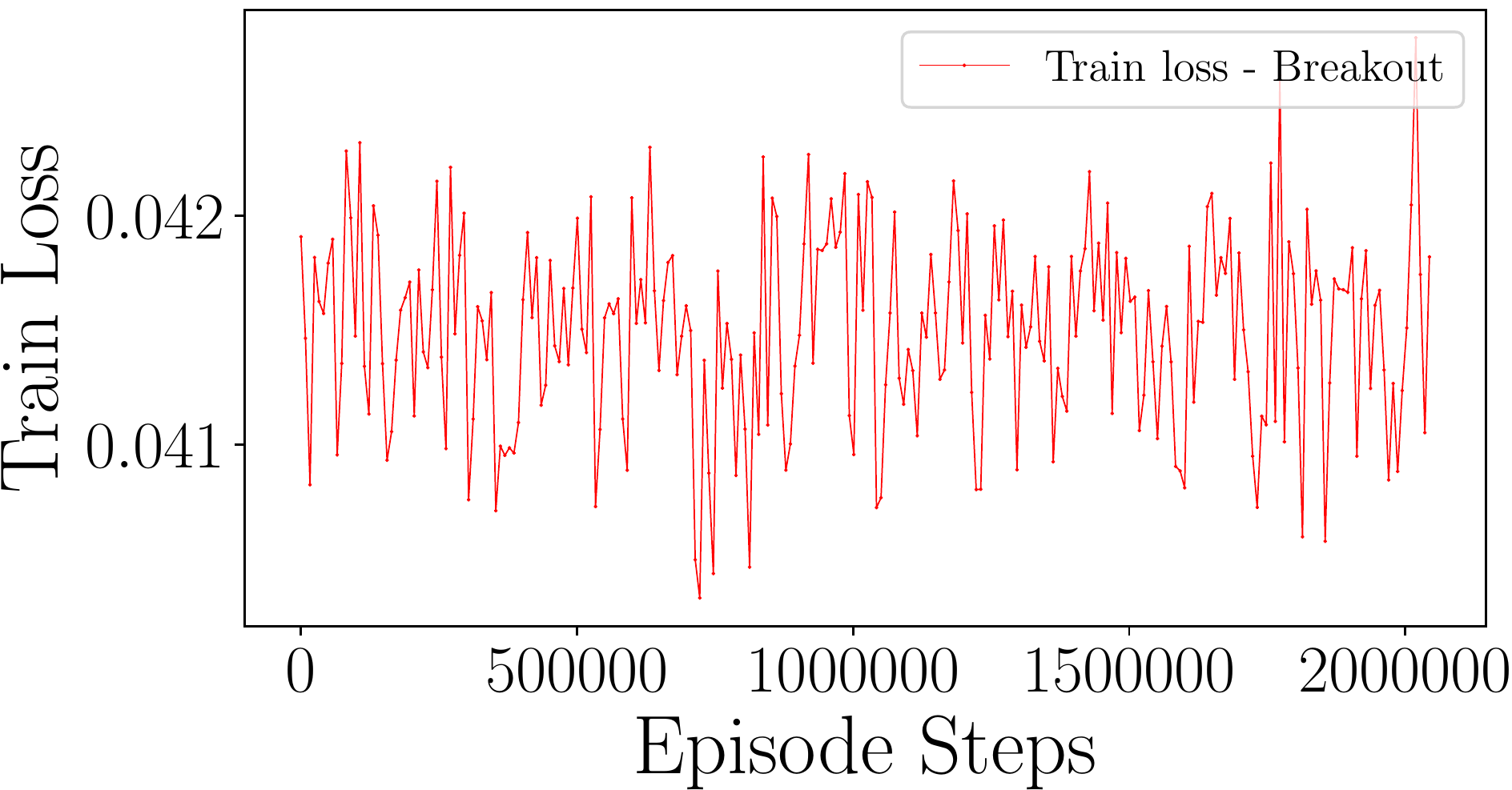} & 	\includegraphics[width=.3\textwidth]{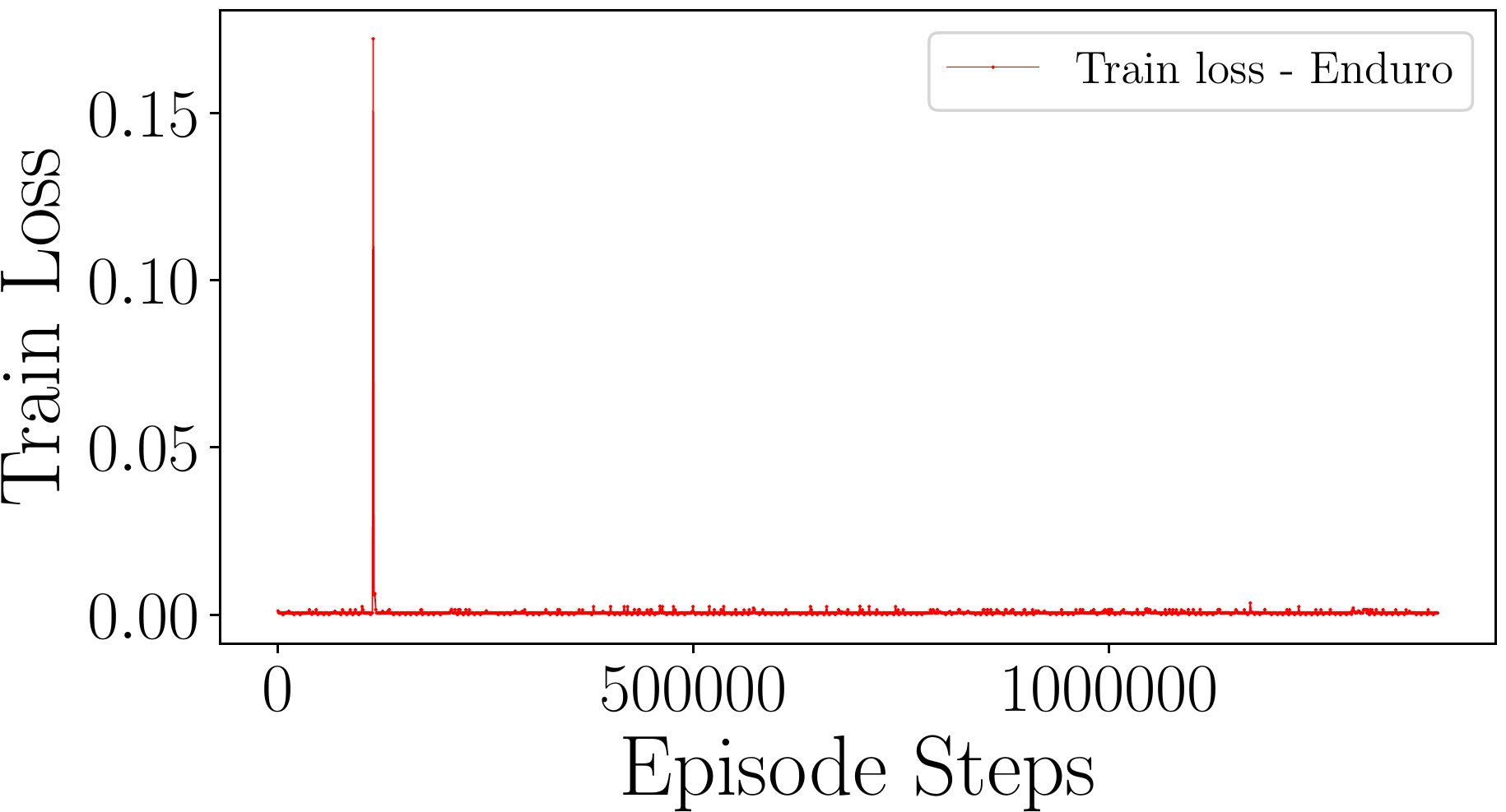} \\
		(j) & (k) & (l)\\
		\includegraphics[width=.3\textwidth]{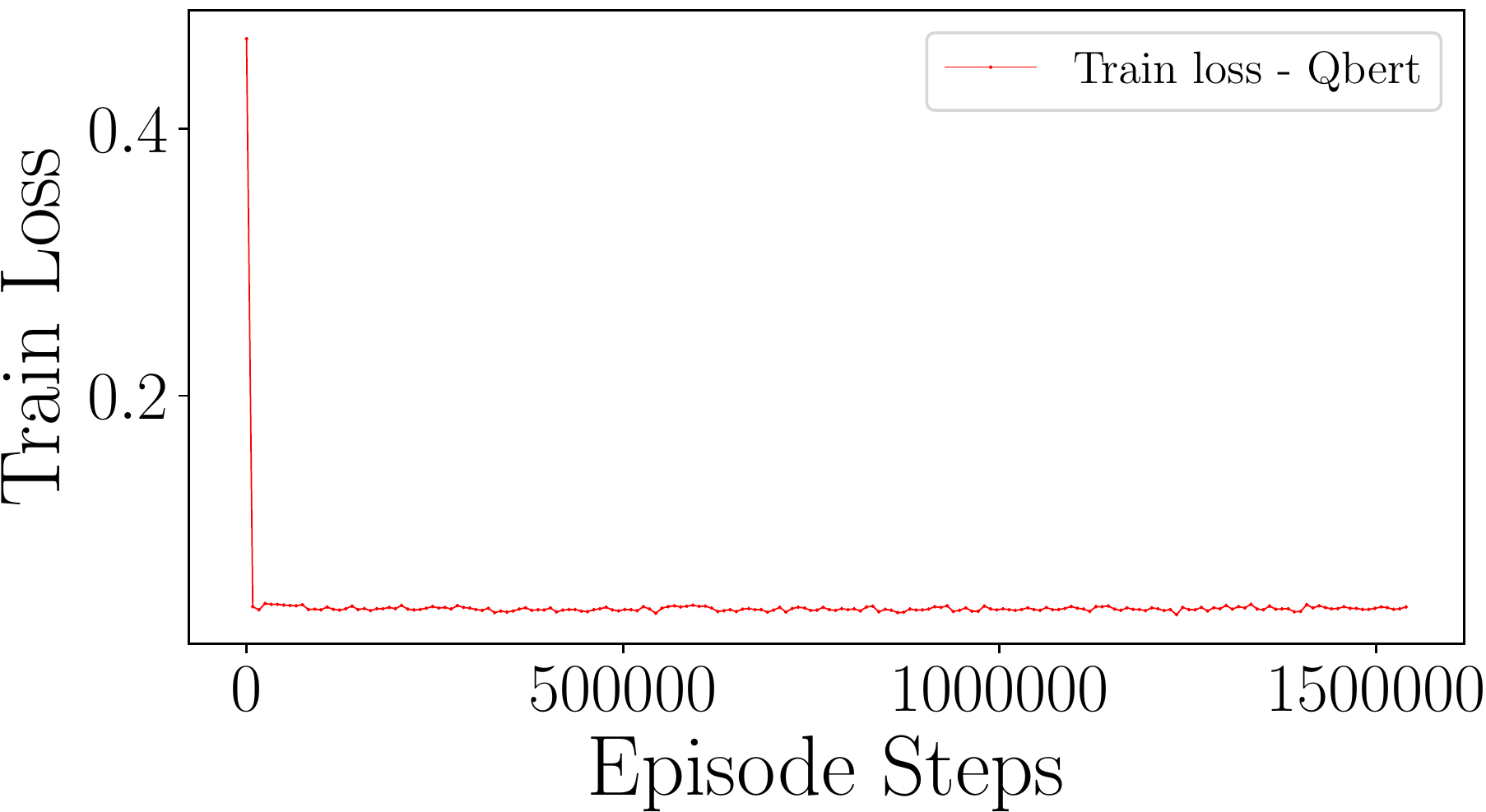} & 
		\includegraphics[width=.3\textwidth]{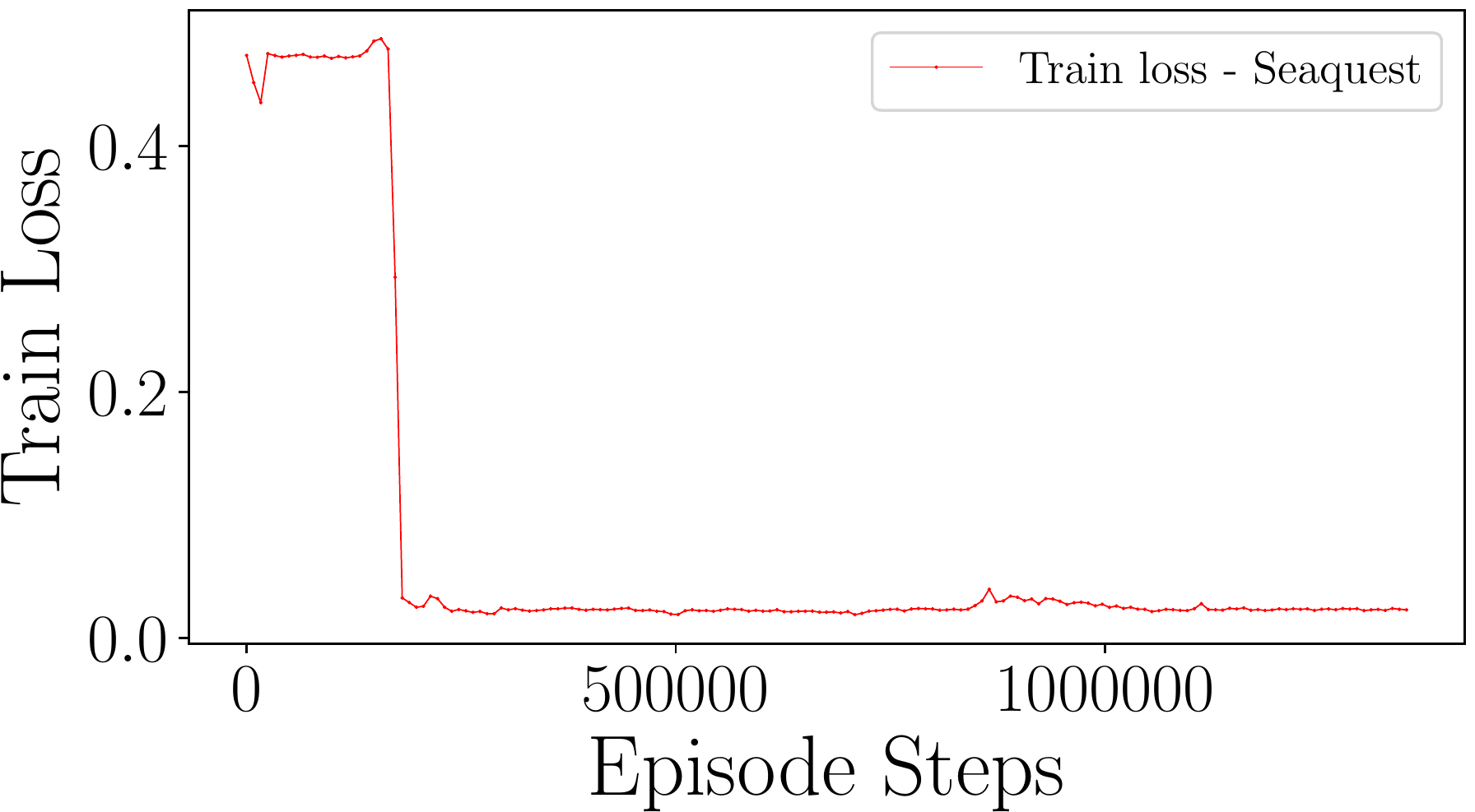} & 	\includegraphics[width=.3\textwidth]{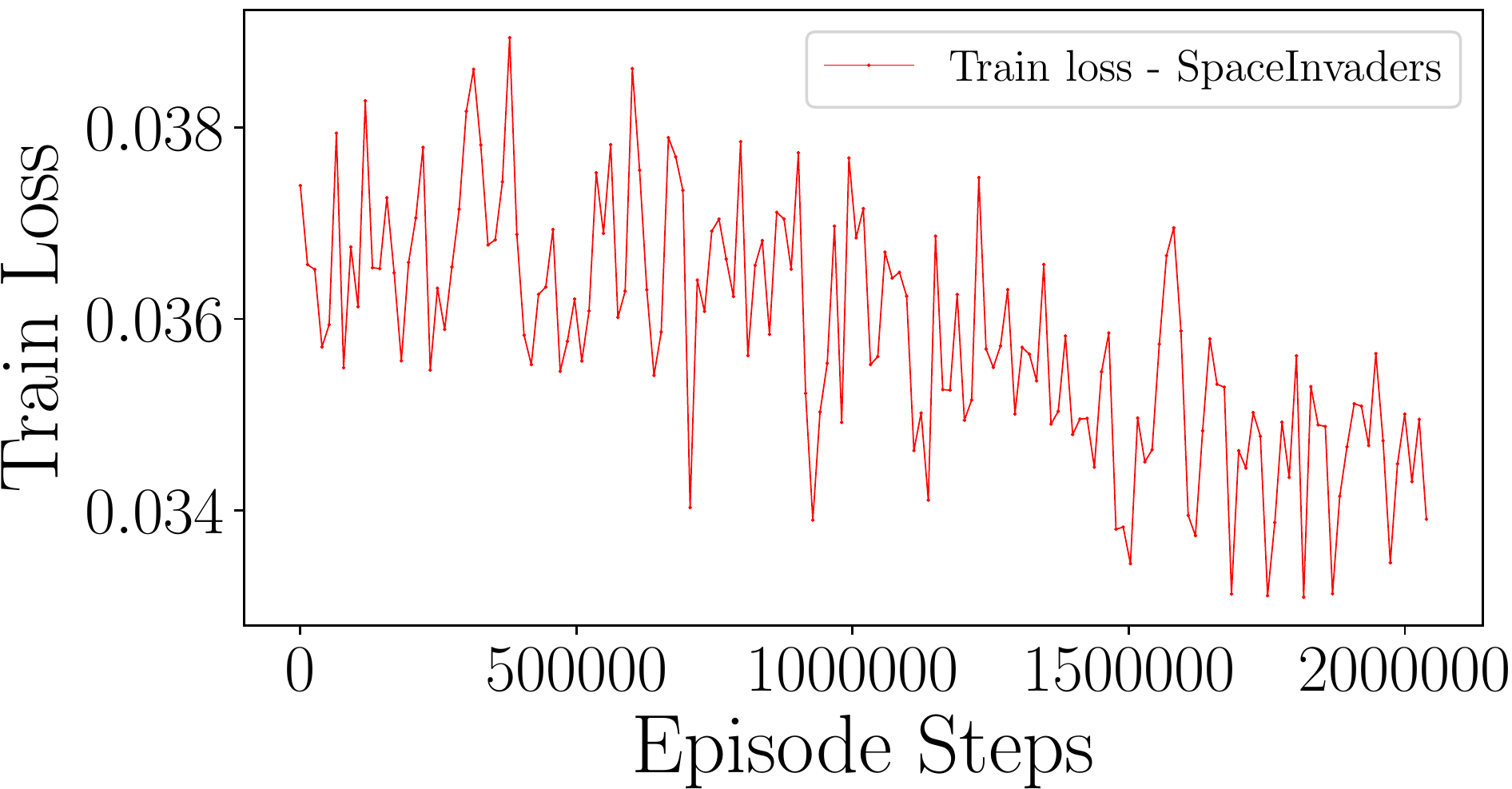} 
	\end{tabular}
	\caption{(a) -- (f) Test scores and (g) -- (l) training loss for six ATARI games --- Beam Rider, Breakout, Enduro, Q*bert, Seaquest, and Space Invaders. The results are form simulations with batch size $b=2048$ and the L-BFGS memory size $m=40$.}
	\label{fig:score-time-ATARI-Games}
\end{figure*}

The results of the Deep L-BFGS Q-Learning algorithm is summarized in Table \ref{t:summary}, which also includes an expert human performance and some recent model-free methods: the Sarsa algorithm \citep{Bellemare:2013:ALE}, the \emph{contingency aware} method from \citep{Bellemare:2012:Contingency}, deep Q-learning \citep{DeepMind:Atari:2013}, and two methods based on policy optimization called Trust Region Policy Optimization (TRPO vine and TRPO single path) \citep{Schulman:2015:TRPO-ATARI}. Our method outperformed most other methods in Space Invaders game. Our deep L-BFGS Q-learning method consistently achieved reasonable scores in the other games. Our simulations only was trained on about 2 million steps (much less than other methods). The training time for our simulations were in the order of 3 hours, which outperformed all other methods. For example, 500 iterations of the TRPO  algorithm took about 30 hours \citep{Schulman:2015:TRPO-ATARI}.  
\renewcommand{\arraystretch}{1.4}
\begin{table}[hbt!]
	\small
	\centering
	\caption{Game Scores for ATARI 2600 Games with different learning methods. Beam Rider (BR), Breakout (BO), Enduro (EO), Q*bert (Q*B), Seaquest (SQ), and Space Invaders (SI)}
	\begin{tabular}[t]{c|cccccc} 
		\hline
		\textbf{Method} &  \textbf{BR} & \textbf{BO} &  \textbf{EO} & \textbf{Q*B} & \textbf{SQ} & \textbf{SI} \\ \hline 
		Random & 354 & 1.2 &0 &157 & 110 & 179 \\ \hline 
		Human & 7456 & 31 & 368 & 18900 & 28010 & 3690  \\ \hline 
		Sarsa
		\citep{Bellemare:2013:ALE}&
		996 & 5.2 &
		129 &
		614 &
		665 &
		271 \\ \hline 
		Contingency \citep{Bellemare:2012:Contingency} &
		1743 & 6 &
		159 &
		960 &
		723
		& 268	\\ \hline
		HNeat Pixel \citep{Hausknecht:2014:HNeat-ATARI}&
		1332 & 4
		& 91
		& 1325 &
		800 &
		1145 \\ \hline
		DQN \citep{DeepMind:Atari:2013} & 4092 &168 & 470 & 1952 &1705 & 581 \\ \hline  
		TRPO, Single path \citep{Schulman:2015:TRPO-ATARI} & 1425 & 10 & 534 & 1973 & 1908 &568 \\ \hline 
		TRPO, Vine \citep{Schulman:2015:TRPO-ATARI} & 859 & 34 & 431 & 7732 & 7788 & 450\\ \hline 
		Our method & 1380 & 18 & 49 &1525 & 600 & 955 \\\hline
	\end{tabular}
	\label{t:summary}
\end{table}
\renewcommand{\arraystretch}{1.0}
\section{Future Work}
For future work, we will consider the optimization methods based on trust-region methods. The trust-region methods require choosing some hyperparameters, like proper initial trust region radius, but they might converge faster than line search since they do not require satisfying the curvature conditions, and sufficient decrease. Additionally the trust-region algorithm can shrink or expand the trust-region radius based on the quality of the search step. Also, since the true Hessian is indefinite, using indefinite quasi-Newton matrices, like Symmetric Rank 1 (SR1), or Full Broyden Class (FBC) within trust-region methods might lead to better convergence properties. We will study these methods as a future work. 
\section{Conclusions}
\label{sec:conclusions}
We proposed and implemented a novel optimization method based on line search limited-memory BFGS for deep reinforcement learning framework. We tested our method on six classic ATARI 2600 games. The L-BFGS method attempts to approximate the Hessian matrix by constructing  positive definite matrices with low-rank updates. Due to the nonconvex and nonlinear loss function in deep reinforcement learning, our numerical experiments show that using the curvature information in computing the search direction leads to a more robust convergence. Our proposed deep L-BFGS Q-Learning method is designed to be efficient for parallel computations in GPU. Our method is much faster than the existing methods in the literature, and it is memory efficient since it does not need to store a large experience replay memory.

\newpage 
\bibliography{draft}
\bibliographystyle{apalike}

\end{document}